\documentclass{article}
\usepackage{natbib}
\usepackage[noend]{algpseudocode}
\usepackage[english]{babel}
\usepackage[utf8]{inputenc}
\usepackage{algorithm}
\usepackage[noend]{algpseudocode}
\usepackage{amsmath}
\usepackage{amsfonts}
\usepackage{amssymb}
\usepackage{graphicx}
\usepackage{todonotes}
\usepackage{pgfplots}
\usepackage{hyperref}
\usepackage{booktabs,siunitx}

\usepackage[toc,page,header]{appendix}
\usepackage{etoc}
\etocdepthtag.toc{mtchapter}
\etocsettagdepth{mtchapter}{subsection}
\etocsettagdepth{mtappendix}{none}

\usepackage[T1]{fontenc}
\usepackage[font=small,labelfont=bf,tableposition=top]{caption}

\DeclareCaptionLabelFormat{andtable}{#1~#2  \&  \tablename~\thetable}

\newcommand{\eps}{\varepsilon}

\newcommand{\calY}{\mathcal{Y}}

\newcommand{\calA}{\mathcal{A}}
\newcommand{\calD}{\mathcal{D}}
\newcommand{\calZ}{\mathcal{Z}}

\newcommand{\reals}{\mathbb{R}}

\usepackage{amsthm}

\newtheorem{theorem}{Theorem}[section]
\newtheorem{lemma}[theorem]{Lemma}
\newtheorem{corollary}[theorem]{Corollary}
\newtheorem{definition}[theorem]{Definition}

\usepackage{amsmath,algorithm,tabularx}
\makeatletter
\newcommand{\multiline}[1]{%
  \begin{tabularx}{\dimexpr\linewidth-\calALG@thistlm}[t]{@{}X@{}}
    #1
  \end{tabularx}
}
\makeatother

\usepackage{blindtext}
\presetkeys%
    {todonotes}%
    {inline,backgroundcolor=yellow}{}%

\newcommand{\ceil}[1]{\left\lceil #1 \right\rceil}

\usepackage{fullpage}

\title{Differentially Private Approximate Quantiles}
\author{Haim Kaplan\thanks{Tel Aviv University and Google Research. \texttt{haimk@tau.ac.il}. Partially supported by Israel Science Foundation (grant 1595/19), German-Israeli Foundation (grant 1367/2017), and the Blavatnik Family Foundation.}
\and
Shachar Schnapp\thanks{Ben-Gurion University. \texttt{schnapp@post.bgu.ac.il}. Partially supported by Israel Science Foundation (grant 1871/19) and by the Cyber Security Research Center at Ben-Gurion University of the Negev.}
\and
Uri Stemmer\thanks{Tel Aviv University and Google Research. \texttt{u@uri.co.il}. Partially supported by Israel Science Foundation (grant 1871/19) and by Len Blavatnik and the Blavatnik Family foundation.}
}

\usepackage{forest}
\usetikzlibrary{calc} %
\newcounter{levelcount} %

\newcommand{\algoname}{\texttt{ApproximateQuantiles}}
\newcommand{\algoref}{\hyperref[alg:bf]{AQ algorithm}}
\newcommand{\nlayer}{\log{m} + 1}

\begin{document}
\date{October 11, 2021}
\maketitle

\begin{abstract}
In this work we study the problem of differentially private (DP) quantiles, in which given dataset $X$  and quantiles $q_1, ..., q_m \in [0,1]$, we want to output $m$ quantile estimations which are as close as possible to the true quantiles and preserve DP. 
We describe a simple recursive DP algorithm, which we call \algoname~(AQ), for this task.
We give a worst case upper bound on its error, and show that its error is much lower than of previous implementations on several different datasets. 
Furthermore, it gets this low error while running time two orders of magnitude faster that the best previous implementation.
\end{abstract}

\section{Introduction}

Quantiles are the values that divide a sorted dataset in a certain proportion. They are one of the most basic and important data statistics, with various usages, ranging from computing the median to standardized test scores \citep{GRE}.
Given sensitive data, publishing the quantiles can expose information about the individuals that are part of the dataset. For example, suppose that a company wants to publish the median of its users' ages. Doing so means to reveal the date of birth of a certain user, thus compromising the user's privacy. \textit{Differential privacy} (DP) offers a solution to this problem by requiring the output distribution of the computation to be insensitive to the data of any single individual. This leads us to the definition of the DP-quantiles problem: 
\begin{definition}[The DP-Quantiles Problem]
Let $a,b\in\reals$.
Given a dataset $X\subseteq (a,b)$ containing $n$ points from $(a,b)$, and a set of $m$ quantiles $0< q_1\leq\dots\leq q_m < 1$, privately identify quantile estimations $v_1, ..., v_m$ such that for every $j\in[m]$ we have $\Pr_{x \sim U_X} [x \leq v_j ] 	\approx q_j$,\footnote{We will make this precise later.} where $U_X$ is the uniform distribution over the data $X$.\footnote{For simplicity, we assume that there are no duplicate points in $X$.}
\end{definition}

On the theory side, the DP-quantiles problem is relatively well-understood, with advanced constructions achieving very small asymptotic error \citep{BeimelNS16,bun2015differentially,kaplan2020privately}. However, as was recently observed by \cite{gillenwater2021differentially}, due to the complexity of these advanced constructions and their large hidden constants, their practicality is questionable. This led \cite{gillenwater2021differentially} to design a simple algorithm for the DP-quantiles problem that performs well in practice. %
In this work we revisit the DP-quantiles problem. We build on the theoretical construction of \cite{bun2015differentially}, and present a new (and simple) practical algorithm that obtains better utility and runtime than the state-of-the-art construction of \cite{gillenwater2021differentially} (and all other exiting implementations).

\subsection{Our Contributions}
We provide \algoname, an algorithm and implementation for the DP-quantiles problem (Section~\ref{sec:algo}). We prove a worst case  bound on the error of the \algoref~for
arbitrary $m$ quantiles, and a tighter error bound for the case 
 of uniform quantiles $q_i = i/(m+1)$, $i=1,\ldots,m$ (Section~\ref{sec:math_error}). We experimentally evaluated the \algoref~and conclude that it obtains higher accuracy and faster runtime than the existing state-of-the-art implementations (Section~\ref{sec:experimental}).
In addition, we adapt  our algorithm (and its competitors) to the definition of Concentrated Differential Privacy (zCDP) \cite{bun2016concentrated}.
We show that its dominance over other methods is even more significant with this definition of privacy.

\subsubsection{A Technical Overview of Our Construction: Algorithm \algoname}

Our algorithm operates using a  DP algorithm for estimating a {\em single} quantile. Specifically, we assume that we have a DP algorithm
$\calA: \reals^2 \times \reals^n \times (0,1) \to \reals$ 
that takes a domain $I=(a,b)\in \reals^2$ (an interval on the line), 
a database $X\in \reals^n$ (containing $n$ points in $I$), and a single quantile $q\in(0,1)$, and returns a point $v\in I$ such that $\Pr_{x \sim U_X} [x \leq v ] 	\approx q$. Estimating a single quantile is a much easier task, with a very simple algorithm based on the exponential mechanism \citep{mcsherry2007mechanism}.

A naive approach of using $\calA$ for solving the DP-quantiles problem would be to run it $m$ times (once for every given quantile). However, due to composition costs of running $m$ DP algorithms on the same data, the error with this approach would scale polynomially with $m$. As we demonstrate in our experimental results (in Section \ref{sec:experimental}), this leads to a significantly reduced performance. In contrast, as we explain next, in our algorithm the error scales only logarithmically in the number of quantiles $m$.

The \algoref~privately estimates (using $\calA$)  the ``middle quantile'' $p=q_{\ceil{m/2}}$ and observes an answer $v$. Then it  splits the problem into two sub-problems. The first sub-problem is defined over the dataset $X_\ell$ which contains the values from $X$ that are smaller than $v$. Its goal is to privately compute the quantiles $(q_1,.., q_{\ceil{m/2} - 1})/p$ on  $X_\ell$. The second problem is defined over the dataset $X_u$ which contains the values from $X$ that are greater than $v$. The goal of the second problem is to privately compute $(q_{\ceil{m/2} + 1} - p, \dots, q_m -p) / (1 - p)$ on $X_u$. 
Notice that the recursive sub-problems have smaller ranges.

Specifically, at the first level of the recursion tree we compute one quantile $q_1^1 \triangleq q_{\ceil{m/2}}$ on data points from a range $(a, b)$.
We denote by $v_1^1$ the estimate which we receive.
At the second level we compute two normalized quantiles $q_2^1 \triangleq q_{\ceil{m/4}}/q_1^1$ and $q_2^2 \triangleq q_{\ceil{3m/4}}/(1 - q_1^1)$. The 
quantile $q_2^1$ 
 is computed on data in the range $(a, v_1^1)$, and we denote its estimate by 
 $v_2^1$. The second quantile
is computed on the data  in the range $[v_1^1, b]$,
and we denote its estimate by 
 $v_2^1$, and so on. Figure~\ref{fig:tree} illustrates this recursion tree.

\begin{figure}[H]
    \centering
    \caption{The
    recursion tree of the algorithm. At each node we write the range of the corresponding subproblem and above it the part of the data in this subproblem.}
    \label{fig:tree}
    
\begin{forest}
[{$(a, b)$} , name=root, label=${X=X_1^1}$ %
[{$(a, v_1^1)$}, name=level1, label=$X_2^1$ %
[{$(a, v_2^1)$}, name=level2, label=$X_3^1$ [$\vdots$]][{$[v_2^1, v_1^1]$},label=$X_3^2$ [$\vdots$]]] %
[{$[v_1^1, b]$}, label=$X_2^2$
[{$[v_1^1, v_2^2]$},label=$X_3^3$[$\vdots$]][{$[v_2^2, b]$},label=$X_3^4$[$\vdots$]]]]
{%
}
\end{forest}

\end{figure}
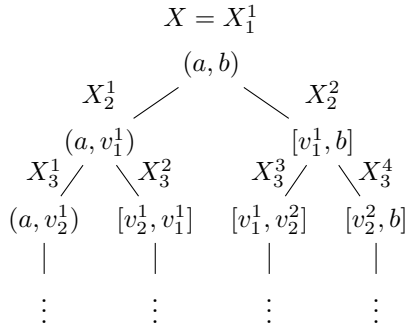

A few remarks are in order.
\begin{enumerate}
    \item By shrinking the data range from one level to the next, we effectively reduce the error of algorithm $\calA$ (because its error depends on the data range). We have found that, in practice, this has a large impact on our accuracy guarantees.
    \item Note that every single data point participates only in $\log(m)+1$ sub-problems (one at each level). This allows our privacy loss (and, hence, our error) to grow only logarithmically in $m$.
\end{enumerate}

\subsection{Related Work}
Private (Comulative Distribution Function) CDF estimation can be applied to estimating all the quantiles \citep{bun2015differentially, kaplan2020privately}, however the theoretically best known algorithm for private CDF estimation \citep{bun2015differentially} relies on several reductions, thus limiting its practicality. We present our \algoref, inspired by \cite{bun2015differentially}, taking their CDF estimation algorithm into practice. As opposed to the algorithm proposed by \cite{bun2015differentially}, our algorithm avoids discretization of the domain by solving the DP single quantile problem using  the exponential mechanism \citep{smith2011privacy} instead of an interior point algorithm \citep{bun2015differentially}. Moreover, we split the data according the desired quantiles (rather than uniformly) while avoiding using the laplace mechanism in the process. A common tree-based approach to CDF estimation is included in our empirical error analysis (Section~\ref{sec:experimental}).
A recent work by \cite{gillenwater2021differentially} proposed an instance of the exponential mechanism that simultaneously draws $m$ quantiles. The naive implementation of this exponential mechanism for $m$ quantiles is computationally difficult, but  \cite{gillenwater2021differentially}  provide a
sophisticated
 $O(mn \log{n} + m^2n)$ implementation. The empirical experiments of \cite{gillenwater2021differentially} show that when the number of quantiles is small, $\mathrm{JointExp}$ algorithm preforms best. A comparison with this algorithm is included in our experiments (Section~\ref{sec:experimental}).

\section{Preliminaries}
\label{sec:prelims}
A database $X$ is a set of $n$ points from 
 some data domain $\calD$.\footnote{The domain in this paper is the interval $(a,b)$.} Differential privacy uses the definition of \emph{neighbors} as follows.
\begin{definition}
    Databases $X$ and $X' \in \calD^n$ are \emph{neighbors}, denoted $X \sim X'$, if one of them can be obtained from the other by adding or removing a single element.
\end{definition}

We use the \emph{add-remove} definition of neighbors, as opposed to the \emph{swap} definition (in which we replace a point instead of deleting or inserting it), although it is important to note that our algorithm \algoname~(AQ) easily adapts to the swap framework.

Differential privacy can be defined using the notion of neighboring databases as follows:
\begin{definition}[\citet{dwork2006calibrating}]
    $\calA$ randomized algorithm $\calA \colon \calD^* \to \calY$ is $(\eps,\delta)$-differentially private ($(\eps,\delta)$-DP) if, for every pair of neighboring databases $X, X'$ and every output subset $Y \subseteq \calY$, 
    \[
    \Pr_{\calA}[A(X) \in Y] \leq e^\eps \cdot \Pr_{\calA}[A(X') \in Y] + \delta.
    \]
    When $\delta > 0$, we say $\calA$ satisfies \emph{approximate} differential privacy. If $\delta = 0$, we say $\calA$ satisfies \emph{pure} differential privacy, and shorthand this as $\eps$-differential privacy ($\eps$-DP).
\end{definition}

The composition property is a key benefit of differential privacy: an algorithm that relies on differentially private subroutines inherits an overall privacy guarantee by simply adding up the privacy guarantees of its components.
\begin{lemma}[\citet{DworkKMMN06,dwork2006calibrating}]
\label{lem:comp}
    Let $\calA_1, \ldots, \calA_k$ be $k$ algorithms that respectively satisfy $(\eps_1, \delta_1)$-$, \ldots, (\eps_k, \delta_k)$-differential privacy. Then running $\calA_1, \ldots, \calA_k$ satisfies $\left(\sum_{i=1}^k \eps_i, \sum_{i=1}^k \delta_i\right)$-differential privacy.
\end{lemma}

The above lemma has come to be known as ``basic composition'' in the literature of differential privacy (see \cite{composition} for more advanced composition theorems). Given Lemma~\ref{lem:comp}, 
a simplistic approach for the DP-quantiles problem would be
to estimate each of the $m$ quantiles separately, using an  $\frac{\eps}{m}$-DP algorithm, and then to apply Lemma~\ref{lem:comp} in order to show that the $m$ estimations together satisfy $\eps$-DP. As we will see, our algorithm outperforms this simplistic approach by a large gap.

We will also rely on the exponential mechanism, a common building block for differentially private algorithms.
\begin{definition}[\citet{mcsherry2007mechanism}]
\label{def:exp}
    Given utility function $u \colon \calD^* \times S \to \mathbb{R}$ mapping $(\text{database}, \text{output})$ pairs to real-valued scores with $L_1$ sensitivity
    \begin{equation*}
    \Delta_u \triangleq \max_{X \sim X', s \in S} |u(X, s) - u(X', s)|,
    \end{equation*}
    the probability that the exponential mechanism $M$ has an output $s \in S$ is:
    \[
    \Pr[M(X) = s] \propto \exp\left(\frac{\eps u(X,s)}{2\Delta_u}\right),
    \]
    where $\propto$ elides the normalization factor.
\end{definition}
The exponential mechanism maintains the database's privacy while prioritizing its higher utility outputs. 
\begin{lemma}[\citet{mcsherry2007mechanism}]
\label{lem:exp_dp}
The mechanism described in Definition~\ref{def:exp} is $\eps$-DP and for error parameter $\gamma>0$ satisfies that:
\[
\Pr_{s \in S}[u(X, s) > \mathrm{Opt}-\gamma] \leq |S|\cdot \exp\left(-\frac{\eps \gamma}{2 \Delta}\right),
\]
where $\mathrm{Opt} = \max_{s \in S} \{u(X, s)\}$.
\end{lemma}

In our case in this paper, the solution space $S$ is infinite. Specifically, it is an interval $(a,b)$. Lemma \ref{lem:expu} in
Appendix \ref{app:exp_dp} 
adapts the exponential mechanism to this setting.

\section{\algoname~Algorithm}\label{sec:algo}

This section demonstrates our differentially private  quantiles algorithm, \algoname. As we mentioned in the introduction, our algorithm uses a subroutine $\calA$ for privately estimating a single quantile. We implement algorithm $\calA$ using the exponential mechanism (see Appendix~\ref{app:exp_dp}).
We remark that, if the dataset is given sorted, then $\calA$ runs in linear time, and the time required for all recursive calls at the same level of the recursion tree is $O(n)$. It follows that the overall time complexity of \algoref~is $O(n\log{m})$.

\begin{algorithm}

\caption{- \algoname~(AQ) Differentially Private Quantiles}
\label{alg:bf}
\textbf{Input:} Domain $\calD=(a,b)$, database $X = (x_1, \dots, x_n) \in \calD^n$, quantiles $Q = (q_1, \dots, q_m)$, %
privacy parameter $\eps$.

\begin{algorithmic}[1]
\State Let 
$\calA: \reals^2 \times \reals^n \times (0,1) \to \reals$ 
be a
$\frac{\eps}{\nlayer}$-DP mechanism for a single quantile.
\Function{F}{$(a , b), X, Q$}
\If {$m = 0$} \textbf{return} $null$
\ElsIf {$m = 1$} 
     \State \textbf{return} $\{\calA((a,b), X,q_1)\}$
\EndIf
\State Let $\hat{m} = \ceil{m/2}$
\State Let $p = q_{\hat{m}}$
\State Let $v = \calA((a,b),X,p)$
\State Let $X_\ell, X_u   = \{x \in X \mid  x < v \},  \{x \in X \mid   x > v \}$
\State Let $Q_\ell =  (q_1 , \dots, q_{\hat{m} - 1})/ p$
\State Let $Q_u =  (q_{\hat{m} + 1} - p, \dots, q_m -p) / (1 - p)$
\State \textbf{return} $\Call{F}{(a, v),X_\ell, Q_\ell} \cup \{v\} \cup \Call{F}{(v, b),X_u, Q_u}$ 
\EndFunction
\end{algorithmic}
\end{algorithm}

We denote by 
$q_i^j$ the normalized quantile computed by the $j$th subproblem at the $i$th level of the algorithm's recursion tree, and by $v_i^j$ the result of this computation. Note that the number of subproblems at level $i$ is $2^{i-1}$. (At the last level some of the subproblems may be empty.) We let $X_{i}^{j}$ be the data used to compute $v_i^j$. It was created by 
 splitting the data $X_{i-1}^{\ceil{j/2}}$ into two parts according to $v_{i-1}^{\ceil{j/2}}$.
 We note that  $X_i^{j_1}$ and 
 $X_i^{j_2}$ are disjoint for a fixed level $i$ and $j_1\not= j_2$.
 This allows us  to avoid splitting our privacy budget between subproblems at the same level of the recursion, and we split it only between different levels, see Section \ref{sec:privacy}.
We also denote 
 $Q_i \triangleq (q_i^1,\dots, q_i^{2^{i-1}})$
and 
  $V_i \triangleq (v_i^1,\dots, v_i^{2^{i-1}})$.

We assume that the data does not contain duplicate points.
This can be enforced by adding small perturbations to the points.
The answer we return is with respect to the 
perturbed points. In fact the utility of our algorithm depends on the minimum distance between a pair of points.

\subsection{Privacy Analysis} \label{sec:privacy}
First we prove that \algoref~is $\eps$-DP.

\begin{lemma}[Differential Privacy]
\label{lem:dp}
If $\cal A$ is an $\frac{\eps}{\nlayer}$-DP mechanism for a single quantile then
\algoref is $\eps$-DP.
\end{lemma}

\begin{proof}
It suffices to show that
for each level $1\le i\le \nlayer$ the output $V_i$ 
 is $\frac{\eps}{\nlayer}$-DP, 
 since the number of levels is $\nlayer$, from composition (Lemma~\ref{lem:comp}) we get that the \algoref~satisfies $(\eps, \delta)$-DP. 
 
 Let $X$ and $X' = X \cup \{x'\}$ be neighboring databases, mark as ${X'}_{i}^{k}$ the part of level $i$ that contains $x'$ among ${X'}_{i}^{j}$, $1\le j \le 2^{i-1}$ (as explained above, only one ${X'}_{i}^{j}$ contains $x'$). For each $j\neq k$ the data $X_{i}^{j}$ equals ${X'}_{i}^{j}$ and therefore the probability of the output $v_{i}^{j}$ is the same under $X$ or $X'$. The output $v_{i}^{k}$ is obtained by $\calA$ which is a $\frac{\eps}{\nlayer}$-DP mechanism, therefore it satisfies $\frac{\eps}{\nlayer}$-DP.

\end{proof}

\subsection{Utility Analysis}\label{sec:math_error}
A $q_i$-quantile is any point $o_i\in (a,b)$ such that the number points of $X$ which are in $(a,o_i)$ is $\lfloor q_i n \rfloor$.
We also define the {\em gap},
 $\mathrm{Gap}_{X}(d_1, d_2)$, between $d_1, d_2 \in (a,b)$ with respect to $X$ are the number of points in the data $X$ that fall between $d_1$ and $d_2$, formally:
\[
\mathrm{Gap}_{X}(d_1, d_2) = |\{x \in X \mid  x \in [\min\{d_1, d_2\} , \max\{d_1, d_2\})\}|.
\]
Using this notion we  define the error of the algorithm. Given dataset $X$, quantiles $Q=(q_1, \dots, q_m)$, solution $V = (v_1, \dots , v_m)$ and  true quantiles $O = (o_1, \dots , o_m)$, the {\em maximum missed points error} is defined as:
\[
\mathrm{Err}_X(O,V) = \max_{j \in [m]} \{ \mathrm{Gap}(o_j, v_j) \}= \max_{j \in [m]} \{ ||\{x \in X | x < v_j\}| - \lfloor q_j \cdot n \rfloor | \}.
\]

This error was first defined by \cite{smith2011privacy} and is widely used in the literature on the differentially private quantiles problem.
We first analyze the error of  our algorithm in the general case, without any constraints on the given quantiles.

\begin{lemma}[General Quantiles Utility]
\label{lem:gqu}
Let $X\in (a,b)^n$ be a database, and 
let $\calA$ be an algorithm that computes an approximation
$v$ for a single quantile $q$ of $X$ such that 
 \[
 \Pr[\mathrm{Gap}_X(o, v) > \gamma] \leq  \frac{\beta}{m}.
 \]
 for some constants $\beta, \gamma > 0$, where
 $o$ is a true $q$-quantile. 
We run 
 \algoref\ using $\calA$  on a database $X$, and quantiles
 $Q = (q_1, \dots , q_m)$. Then, with probability $1 - \beta$, we get approximate quantiles $V = (v_1, \dots, v_m)$ such that $\mathrm{Err}_X(O, V) \leq (\nlayer) \gamma$.
\end{lemma}

\begin{proof}
For the computation of $m$ 
quantiles the \algoref\
applies  $\calA$
at most $m$ times (once per each internal node of the recursion tree).
Since in 
each run 
$\calA$ has error  at most $\gamma$ with probability $1 - \beta/m$ it follows by a 
 union bound that:

\begin{equation} \label{eq:1}
\Pr[\exists i,j  \text{ s.t.\ } \mathrm{Gap}_{X}(\hat{o}_i^j, v_i^j) > \gamma] \leq  \beta,
\end{equation}
where $\hat{o}_i^j$ is a true $q_i^j$-quantile with respect to the dataset $X_i^j$. We also denote by $o_i^j$ a true  $q_k$-quantile 
with respect to $X$ where $q_k$
is the original fraction in $Q$ that corresponds to 
 $q_i^j$ (see Figure~\ref{fig:ranges}). 

At the first level $(i=1)$ of the recursive tree we compute one quantile $q_1^1$ on the data $X = X_1^1$ and therefore $ \mathrm{Gap}_X(\hat{o}_1^1, o_1^1) = 0$. At the second leval $(i=2)$, we split the data $X$ according to $v_1^1$ into $X_2^1, X_2^2$, since $\mathrm{Gap}_X(v_1^1, o_1^1) \leq \gamma$, the
$q_2^j$-quantile
 $\hat{o}_2^j$  of the dataset $X_2^j$ satisfies that $\mathrm{Gap}_{X}(\hat{o}_2^j, o_2^j) \leq \gamma$ (see Figure~\ref{fig:ranges}). By induction, at layer $i$, for  every $j$ we have that $\mathrm{Gap}_{X}(\hat{o}_i^j, o_i^j) \leq (i - 1) \cdot \gamma$. Combining this with Equation~(\ref{eq:1}) results in $\mathrm{Gap}_{X}(v_i^j, o_i^j)  \leq i \cdot \gamma$. At the last level $(i = \nlayer)$
 we have that $\mathrm{Gap}_{X}(v_i^j, o_i^j) \leq  (\nlayer) \cdot \gamma$.
\end{proof}

\begin{theorem}
\label{thm:gen}
Assume that we implement $\calA$ using the exponential mechanism with privacy parameter $\frac{\eps}{\nlayer}$,  as described in Appendix~\ref{app:exp_dp} to solve the single quantile problem. Then, given a database $X \in (a,b)^n$ and quantiles $Q = (q_1, \dots , q_m)$ the \algoref~is $\eps$-DP, and with probability  $1 - \beta$ output $V = (v_1, \dots, v_m)$ that satisfies
\[
\mathrm{Err}_X(O, V) \leq 2(\nlayer) \cdot \frac{\log{\psi} +\log{m} - \log{\beta}}{\eps} = O\left(\frac{\log{m} (\log{\psi} + \log m + \log{\frac{1}{\beta}})}{\eps}\right),
\]
where $\psi = \frac{b-a}{\min_{k \in [n+1]} \{x_k - x_{k-1} \}}$.\footnote{We define $x_0=a$ and $x_{n+1}=b$. These are not real data points.}
\end{theorem}
\begin{proof}
By Lemma~\ref{lem:expu},
 the exponential mechanism with privacy parameter $\frac{\eps}{\nlayer}$ 
has an error more than $2\frac{\log{\psi} +\log{m} - \log{\beta}}{\eps}$ with probability at most $\frac{\beta}{m}$. 
Therefore, by Lemma~\ref{lem:gqu}, the \algoref~has an error no larger than $\gamma=2 (\nlayer) \cdot \frac{\log{\psi} +\log{m} - \log{\beta}}{\eps}$ with probability $1-\beta$.
Combining this with 
Lemma~\ref{lem:dp} the theorem follows.
\end{proof}

\usepgfplotslibrary{groupplots}
\begin{figure}
    \centering
    \caption{An error in computing $v_1^1$ by at most $\gamma = 2$ points from the optimal value $o_1^1$, might cause an error of at most  $\gamma$ points between the value of $o_j^2$ compared to $\hat{o}_j^2$. The example demonstrates the computed quantiles $q_1 = \frac{1}{4}$, $q_2 = \frac{1}{2}$ and $q_3 =\frac{3}{4}$. Note that $q_2^1$, $q_2^2 = \frac{1}{2}$.}
    \label{fig:ranges}

\begin{tikzpicture}
  \begin{axis}[name=plot1,height=2.5cm, width=12cm,
    xmin=0, xmax=22,
    axis x line=bottom,
    hide y axis,    
    ymin= 0, ymax=2,
    scatter/classes={a={mark=o,draw=black}}]
        \node[] at (axis cs: 1,0.45) {$X$};
        \addplot[scatter,only marks, mark size = 2pt, fill = red, scatter src=explicit symbolic] 
            coordinates {(2, 0)  (2.5, 0) (3, 0) (4.5, 0) (7.5, 0) (8, 0) (10, 0) (11, 0) (12, 0) (13, 0) (14, 0) (16, 0) (18, 0) (19, 0) (19.5, 0) (20, 0)};
         \addplot+[black, line width=1.5pt,no marks] coordinates {(9, 0) (9, 0.5)} node[pos=1,above] {\small $v_1^1$};
         
          \addplot+[blue, line width=1.5pt,no marks] coordinates {(11.5, 0) (11.5, 0.5)} node[pos=1,above] {\small $o_1^1 = \hat{o}_1^1$};
          
            \addplot+[blue, line width=1.5pt,no marks] coordinates {(17, 0) (17, 0.5)} node[pos=1,above] {\small $o_2^2$};
        
            \addplot+[blue, line width=1.5pt,no marks] coordinates {(6.5, 0) (6.5, 0.5)} node[pos=1,above] {\small $o_2^1$};
\end{axis} 

\begin{axis}[name=plot2,at={($(plot1.west) + (0,-2cm)$)}, anchor=west, height=2.5cm, width=6cm,
    xmin=0, xmax=10,
    axis x line=bottom,
    hide y axis,    
    ymin= 0, ymax=2,
    scatter/classes={a={mark=o,draw=black}}]
            \node[] at (axis cs: 1,0.45) {$X_1^2$};
        \addplot[scatter,only marks, mark size = 2pt, fill = red, scatter src=explicit symbolic] 
            coordinates {(2, 0)  (2.5, 0) (3, 0) (4.5, 0) (7.5, 0) (8, 0)};
         \addplot+[black,line width=1.5pt,no marks] coordinates {(9, 0) (9, 0.5)} node[pos=1,above] {\small $v_1^1$};
         
         \addplot+[blue, line width=1.5pt,no marks] coordinates {(4, 0) (4, 0.5)} node[pos=1,above] {\small ${\hat{o}}_2^1$};
         
          \addplot+[blue, line width=1.5pt,no marks] coordinates {(6.5, 0) (6.5, 0.5)} node[pos=1,above] {\small $o_2^1$};
\end{axis}

\begin{axis}[name=plot3,at={($(plot1.west) + (6cm,-2cm)$)}, anchor=west, height=2.5cm, width=6cm,
    xmin=6, xmax=22,
    axis x line=bottom,
    hide y axis,    
    ymin= 0, ymax=2,
    scatter/classes={a={mark=o,draw=black}}]
      \node[] at (axis cs: 7,0.45) {$X_2^2$};
      
        \addplot[scatter,only marks, mark size = 2pt, fill = red, scatter src=explicit symbolic] 
            coordinates {(10, 0) (11, 0) (12, 0) (13, 0) (14, 0) (16, 0) (18, 0) (19, 0) (19.5, 0) (20, 0)};
         \addplot+[black, line width=1.5pt,no marks] coordinates {(9, 0) (9, 0.5)} node[pos=1,above] {\small $v_1^1$};
         
         \addplot+[blue, line width=1.5pt,no marks] coordinates {(15, 0) (15, 0.5)} node[pos=1,above] {\small $\hat{o}_2^2$};
         
          \addplot+[blue, line width=1.5pt,no marks] coordinates {(17, 0) (17, 0.5)} node[pos=1,above] {\small $o_2^2$};
\end{axis}

\end{tikzpicture}
\end{figure}
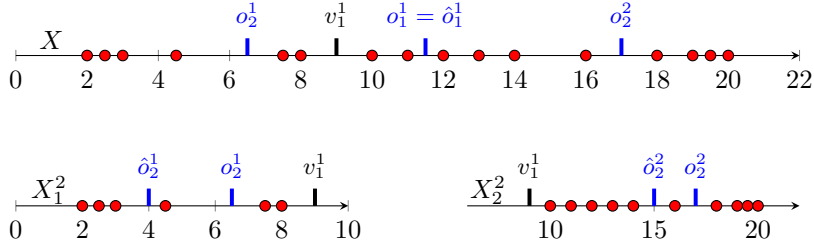

The uniform quantiles problem is a common instance of the quantiles problem where $q_i = \frac{i}{m + 1}$,
for $i=1,\ldots,m$. Lemma~\ref{lem:uqu} shows that when the desired quantiles are uniform it is possible to guarantee
that $\mathrm{Err}_X(O, V) = O(\gamma)$ with probability
 $1 - \beta$.

\begin{lemma}[Uniform Quantiles Utility]
\label{lem:uqu}

Let $X\in (a,b)^n$ be a database, and 
let $\calA$ be an algorithm that computes an approximation
$v$ for a single quantile $q$ of $X$ such that 
 \[
 \Pr[\mathrm{Gap}_X(o, v) > \gamma] \leq  \frac{\beta}{m}.
 \]
 for some constants $\beta, \gamma > 0$, where
 $o$ is a true $q$-quantile. 
We
 run \algoref\ using $\calA$  on a database $X$, and quantiles $Q = (q_1, \dots , q_m)$ where $q_i = \frac{i}{m + 1}$. Then,
with probability  $1 - \beta$, the \algoref~returns approximate quantiles $V = (v_1, \dots, v_m)$ satisfying $\mathrm{Err}_X(O, V) \le 2\gamma$.
\end{lemma}

\begin{proof}
For simplicity we assume that $m = 2^k - 1$.
The proof is similar for 
other values of $m$. 
For this value of $m$ we have that 
$q_\ell = \ell/2^k$, $\ell=1,\ldots , 2^k-1$ and $q_i^j = 1/2$ for all
$i$ and $j$. Furthermore, we have that the number of points in $X\cap (a,o_i^j)$ is
$\lfloor \frac{2\cdot j-1}{2^i} n \rfloor$, $j=1,\ldots, 2^{i-1}$.

As in Lemma \ref{lem:gqu}, 
Equation (\ref{eq:1}) holds. So we assume in the rest of the proof that 
$\mathrm{Gap}_{X}(v_i^j, \hat{o}_i^j) \le \gamma$ for all $i$ and $j$.

For the root $(i=1)$ of the recursion tree we compute one quantile $q_1^1$ on the data $X = X_1^1$ and therefore $\mathrm{Gap}_{X}(\hat{o}_1^1, o_1^1) = 0$. At the second level, $(i=2)$, we split the data $X$ according to $v_1^1$ into $X_2^1$, and $X_2^2$. Since $\mathrm{Gap}_{X}(v_1^1, o_1^1) \leq \gamma$,
we have that
$\lfloor n/2 \rfloor -\gamma\le |X_2^j| \le \lfloor n/2 \rfloor + \gamma$. 
Recall that  $|X\cap (a,o_2^1)| = \lfloor n/4 \rfloor$ and 
$|X\cap (a,o_2^2)| = \lfloor 3n/4 \rfloor$ and therefore  $\mathrm{Gap}_{X}(\hat{o}_2^j, o_2^j) \leq \gamma/2$ for $j=1,2$.
Now, since $\mathrm{Gap}_{X}(v_2^j, \hat{o}_2^j) \le \gamma$ we get that 
$\lfloor n/4 \rfloor -3\gamma/2\le |X_3^j| \le \lfloor n/4 \rfloor + 3\gamma/2$,  and therefore 
$\mathrm{Gap}_{X}(\hat{o}_3^j, o_3^j) \leq 3 \gamma/4$,
for
$j\in [4]$. Similarly,  
since $\mathrm{Gap}_{X}(v_3^j, \hat{o}_3^j) \le \gamma$, it follows that 
$\lfloor n/8 \rfloor -7\gamma/4\le |X_4^j| \le \lfloor n/8 \rfloor + 7\gamma/4$ and therefore $\mathrm{Gap}_{X}(\hat{o}_4^j, o_4^j) \leq 7 \gamma/8$, for $j\in [8]$. We conclude that by induction on the levels we get that,
\[
\mathrm{Gap}_{X}(\hat{o}_i^j, o_i^j)  \leq \frac{2^{i-1} -1}{2^{i-1}} \cdot \gamma \leq \gamma,
\]
for all $i$ and $j$.
Combining this upper bound  with Equation~(\ref{eq:1}) we get that $\mathrm{Gap}_{X}(v_j^i, o_j^i)  \leq 2\gamma$.
\end{proof}

Theorem \ref{thm:uniform}
improves the bound of Theorem
\ref{thm:gen}
for uniform quantiles. We omit its  proof which is similar to the proof of Theorem \ref{thm:gen} using 
Lemma \ref{lem:uqu} instead of 
Lemma \ref{lem:gqu}.

\begin{theorem} \label{thm:uniform} If we
set $\calA$ to be the exponential mechanism with privacy parameter $\frac{\eps}{\nlayer}$,  as described in Appendix~\ref{app:exp_dp},  then given data $X \in (a,b)^n$ and quantiles $Q = (q_1, \dots , q_m)$, where $q_i = \frac{i}{m + 1}$, the \algoref~is $\eps$-DP and with probability  $1 - \beta$ output $V = (v_1, \dots, v_m)$ that satisfies
\[
\mathrm{Err}_X(O, V)  = O\left(\frac{ \log{\psi} + \log m + \log{\frac{1}{\beta}}}{\eps}\right),
\]
where $\psi = \frac{b-a}{\min_{k \in [n+1]} \{x_k - x_{k-1} \}}$.
\end{theorem}

\paragraph{Quantiles sanitization.}
Given a database $X=(x_1,.., x_n)\in (a,b)^n$,
we can produce a differentially private 
dataset $\hat{X}=(\hat{x}_1, \dots, \hat{x}_n)\in (a,b)^n$, such that for each 
point $\hat{x}_\ell$, the number of points in
$\hat{X}$ that are smaller than $\hat{x}_\ell$ is similar to the number of points in $X$ that are smaller than $\hat{x}_\ell$.
This is specified precisely in 
 the following Corollary of
Theorem \ref{thm:gen}. 
In particular this corollary implies that for every interval $I\subseteq (a,b)$, $|X\cap I|$ is approximately equal to
$|\hat{X} \cap I|$.

\begin{corollary} \label{thm:san}
\label{thm:gen}
Assume that we implement $\calA$ using the exponential mechanism with privacy parameter $\frac{\eps}{\log n + 1}$,  as described in Appendix~\ref{app:exp_dp}, to solve the single quantile problem. Then, given a database $X \in (a,b)^n$ and $n$ quantiles $Q = (q_1, \dots , q_n)$, where $q_i=i/n$, the \algoref~is $\eps$-DP, and with probability  $1 - \beta$ output $\hat{X} = (\hat{x}_1, \dots, \hat{x}_n)$ that satisfies
\[
\mathrm{Err}_X(X, \hat{X}) 
=
O\left(\frac{\log{n} (\log{\psi} + \log n + \log{\frac{1}{\beta}})}{\eps}\right).
\]
where $\psi = \frac{b-a}{\min_{k \in [n+1]} \{x_k - x_{k - 1} \}}$.
\end{corollary}

\subsection{Zero-Concentrated Differential Privacy}
\label{sec:zcdp}
Zero Concentrated Differential Privacy (zCDP) \citep{bun2016concentrated} offers smoother  composition properties
than standard $(\eps,\delta)$-DP.
The general idea is to to compare the Rényi divergence of the privacy losses random variables for neighboring databases. We analyse our algorithm also under this definition of privacy. 
As in Section \ref{sec:privacy}, the
 privacy analysis of our algorithm applies composition of the processing of different levels in the recursion tree.  zCDP's composition theorem allows us to run the exponential mechanism
 at each level 
 with a higher privacy parameter, which results in  a tighter error bound for the exponential mechanism. For precise statements see Theorem~\ref{lem:gqu_zcdp} and Theorem~\ref{lem:uqu_zcdp} below. 
 In Section~\ref{sec:accuracy} we measure empirically the benefit the gain from this smoother composition of zCDP.

\begin{definition}[Zero-Concentrated Differential Privacy (zCDP) \citet{bun2016concentrated}]
An algorithm $M:\calD^n \to \mathbb{R}$ is $\rho$-zCDP if for all neighbouring $X,X' \in \calD^n$, and $\gamma \in (1, \infty)$ $\mathrm{RD}_{\gamma}(M(Z),M(Z')) \leq \rho\gamma$, where $\mathrm{RD}_{\gamma}$ is the $\gamma$-Rényi divergence between random variables A and B. ($\cal D$ is the domain of the database elements, in our case it is $(a,b)$.)
\end{definition}

\begin{lemma}\citep{bun2016concentrated}
\label{lem:cdp}
if algorithm $M$ satisfies $\eps$-DP, then $M$ satisfies $\rho$-zCDP with $\rho=\frac{\eps^2}{2}$.
\end{lemma}

\begin{lemma}\citep{bun2016concentrated}
\label{lem:cdp_comp}
Let $M: \calD^n \to \calY$ and $M':  \calD^n \to \calZ$. suppose $M$ satisfies $\rho$-zCDP and $M'$ satisfies $\rho'$-zCDP. Define $M'':\calD^n\to \calZ$ by $M''(X) = M'(X, M(X))$. Then $M''$ satisfies $(\rho+\rho')$-zCDP.
\end{lemma}

\begin{theorem}[General Quantiles Utility with zCDP]
\label{lem:gqu_zcdp}
Suppose we implement $\calA$ using the exponential mechanism to solve the single quantile problem, with privacy parameter  $\eps=\sqrt{\frac{2\rho}{\nlayer}}$.
Then, given data $X$ and quantiles $Q = (q_1, \dots , q_m)$, the \algoref~is $\rho$-zCDP and with probability  $1 - \beta$ output $V = (v_1, \dots, v_m)$ that satisfies
\[
\mathrm{Err}_X(O, V) \leq 2 (\nlayer) \sqrt{\frac{\nlayer}{2\rho}} \cdot (\log{\psi} +\log{m} - \log{\beta}) = O\left(  \frac{\log^{1.5}{m}}{\sqrt{\rho}} (\log{\psi} +\log{m}+ \log{\frac{1}{\beta}}) \right),
\]
where $\psi = \frac{b-a}{\min_{k \in [n+1]} \{x_k - x_{k - 1} \}}$.
\end{theorem}

\begin{proof}
By Lemma~\ref{lem:dp}, the computation at each recursive level $1\le i\le \nlayer$ is $\sqrt{\frac{2\rho}{\nlayer}}$-DP,
and therefore, by
Lemma~\ref{lem:cdp}, also $(\frac{\rho}{\nlayer})$-zCDP. Since the number of levels is $\nlayer$, by Lemma~\ref{lem:cdp_comp}, the \algoref~is $\rho$-zCDP.
The error bound follows exactly as in the proof of Theorem 
\ref{thm:gen}.
\end{proof}

The following theorem is analogous to Theorem \ref{thm:uniform}. Its privacy proof is as for Theorem \ref{lem:gqu_zcdp} and the error analysis is
as in the proof of Theorem \ref{thm:uniform}.

\begin{theorem}[Uniform Quantiles Utility with zCDP]
\label{lem:uqu_zcdp}
Suppose we implement $\calA$ using the exponential mechanism to solve the single quantile problem, with privacy parameter  $\sqrt{\frac{2\rho}{\nlayer}}$. Then, given data $X$ and quantiles $Q = (q_1, \dots , q_m)$, where $q_i = \frac{i}{m + 1}$, the \algoref~is $\rho$-zCDP, and with probability  $1 - \beta$ output $V = (v_1, \dots, v_m)$ that satisfies
\[
\mathrm{Err}_X(O, V) = O\left(\sqrt{\frac{\log{m}}{\rho}} (\log{\psi} +\log{m}+ \log{\frac{1}{\beta}}) \right), 
\]
where $\psi = \frac{b-a}{\min_{k \in [n+1]} \{x_k - x_{k - 1} \}}$.
\end{theorem}

\section{Experiments}\label{sec:experimental}
We implemented the
 \algoref~in Python and its code is publicly available on GitHub. We used the exponential mechanism for the DP single quantile algorithm $\calA$, with 
 $-\mathrm{Gap}_{X}(o,v)$ as the utility of 
a solution $v$, where
$o$ is a true quantile, see
  Appendix~\ref{app:exp_dp}. 
We also experimented with the 
 AQ-zCDP algorithm, a version of our algorithm that is private with respect to  the definition of zero-concentrated differential privacy (zCDP) (Section~\ref{sec:zcdp}).
We compared our algorithms to  the three best performing algorithms from \cite{gillenwater2021differentially} called: (1) JointExp (2) $\mathrm{AppindExp}$ and (3) AggTree.
We ran the implementations provided by \citet{gillenwater2021differentially}.
We describe these baseline algorithms in Section \ref{sec:base}.
We tested the algorithms using two synthetic datasets and four real datasets
that are described in
Section~\ref{sec:datasets}. For each dataset we compared the  accuracy (Section~\ref{sec:accuracy}) and runtime (Section~\ref{sec:time}) of the competing algorithms.

\subsection{Baseline Algorithms} \label{sec:base}
 \textbf{JointExp} \citep{gillenwater2021differentially}  solves the DP quantiles problem
 by an efficient implementation of the exponential mechanism
 (Definition~\ref{def:exp})
 on $m$-tuples $V=(v_1,\dots, v_m) \in (a,b)^m$, where the utility of $V$ is defined as follows:
\[
u(X, V) = -\sum_{j \in [m+1]} | \mathrm{Gap}_X(o_j, o_{j-1})  - \mathrm{Gap}_X(v_j, v_{j-1})\}|,
\]
where we define $v_0=o_0=a$ and $v_{m+1}=o_{m+1}=b$. The naive implementation of 
the exponential mechanism with
this utility function is computationally difficult:
The number of $m$ tuples $V$ is infinite, and 
there may even be exponentially many (in $m$) equivalence classes of such $m$-tuples. %
 \citet{gillenwater2021differentially}  give an $O(mn \log{n} + m^2n)$  time 
 algorithm to sample from the distribution defined by the exponential mechanism.
The experiments of \cite{gillenwater2021differentially} show that when the number of quantiles, $m$, is small the $\mathrm{JointExp}$ algorithm preforms best.
\newline
\newline
\textbf{AppindExp} solves the DP quantiles problem by applying the exponential  mechanism as described in Appendix~\ref{app:exp_dp} to find every quantile $q_i$ separately.
Since $\mathrm{AppindExp}$ applies the exponential mechanism $m$ times, if we use $\eps/m$ as the privacy parameter for each application of the exponential mechanism, then  by composition we get that $\mathrm{AppindExp}$ is $\eps$-DP.

The
advanced composition theorem  \cite{composition} shows
that if we use $\eps' \approx \eps/\sqrt{m\log(1/\delta)}$ for each application of the exponential mechanism then the overall algorithm would be 
$(\eps,\delta)$-DP. The implementation of \cite{gillenwater2021differentially} uses a
 tighter advanced composition theorem specific for nonadaptive applications of the exponential mechanism \citep{dong2020optimal}, to determine 
 an $\eps'$ for each quatile computation such that the overall composition of the $m$ applications is $(\eps,\delta)$-DP.
 We use $n = 1000$ data points in our experiments, so we chose $\delta = 10^{-6}$ in accordance with the common practice that $\delta \ll \frac{1}{n}$.
\newline
\newline
\textbf{AggTree} \cite{dwork2010differential} and
\cite{chan2011private} implement an $\eps$-DP tree-based counting algorithm for CDF estimation. Given a domain $(a,b)$ the algorithm builds a balanced tree $T$ with  branching factor $b$ and height $h$, so $T$ has $b^h$ leaves. The $j$'th leaf of the tree is associated with sub-domain $[c_{j-1}, c_{j}]$ where $c_j := a + j(b-a)/b^{h}$. Given a dataset $D \in (a,b)^n$, the algorithm starts by counting the number of elements from the dataset that fall into each leaf (i.e.\ are contained in its sub-domain).
Each internal node of $T$ is associated with the sum of the counters of its children (which equals to the number of elements in the leaves of its subtree). In particular, the count associated with the root is $n$.
 Since each element in the data contributes to at most $h$ nodes (the path from the leaf containing it to the root), it suffice to add $\mathrm{Lap}(h/\eps)$ noise to the value of each node to make the counts of $T$ $\epsilon$-DP. 
 
 We can approximate any quantile $q$ using this data structure as follows.
 We find the leftmost leaf $z$ such that the sum of the noisy counts of all leaves to the left of $z$ (including $z$) is at least $q\cdot n$.
 In particular, if the counts were not noisy that $z$ would contain a $q$th quantile.
Let $c(z)$ be the noisy count of $z$ and let 
$c^-(z)$ be the sum of the noisy counts of the leaves to the left of $z$.
Let $p=(qn-c^-(z))/c(z)$. Without noise $p$ would have been the approximate relative quantile of the $q$th quantile among the elements in $z$. 
Let $[c_{k-1}, c_{k}]$ be the range associated with $z$.
We approxmate
the $q$th quantile using linear interpolation inside  $[c_{k-1}, c_{k}]$. That is we return 
 $(1 - p) c_{k-1} + p c_{k} $. We utilize the $\mathrm{AggTree}$ implementation provided by \cite{gillenwater2021differentially}, and the results are given in Section~\ref{sec:accuracy}.

\subsection{Datasets}\label{sec:datasets}
We tested our four algorithms on six datasets.
Two data sets are synthetic. One contains independent samples  from the uniform distribution $U(-5,5)$, and the other contains independent samples from the  Gaussian  $N(0,5)$. Two real continuous datasets  from \cite{Goodreads}, one contains book ratings and the other contains books' page counts. Last we have two categorial datasets from the adults' census data  \citep{UCI2019}. One contains working hours per week and the other ages of different persons. Table~\ref{tab:datasets} shows the properties of each dataset, and Figure~\ref{fig:data} shows the histograms of 100 equal-width bins for each dataset.

\begin{center}
  \begin{minipage}{\textwidth}
  \begin{minipage}[b]{0.49\textwidth}
    \label{fig:data}
    \centering
        \includegraphics[width = 0.86\textwidth]{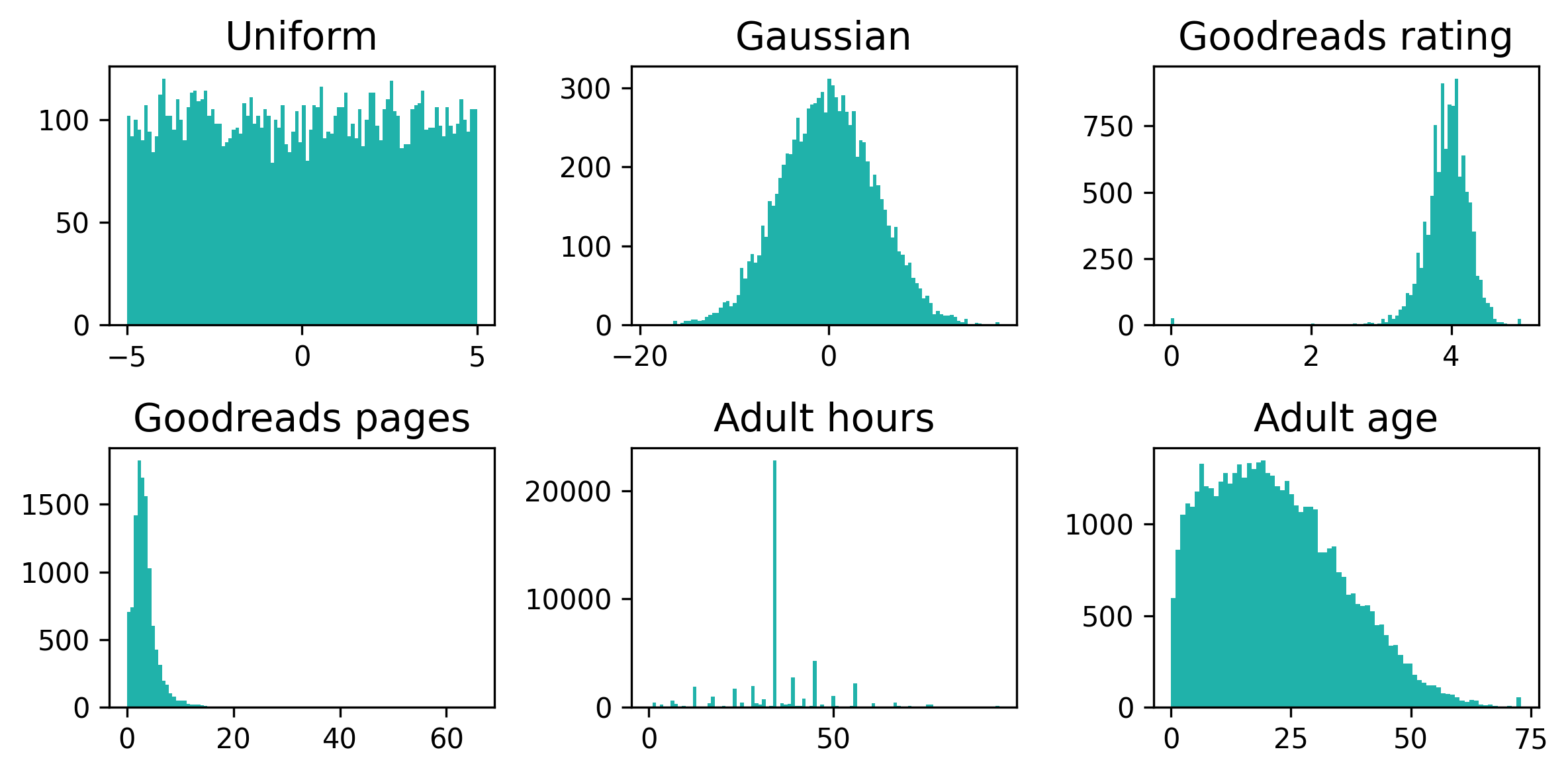}
    \captionof{figure}{Histograms of 100 equal-width bins.}
  \end{minipage}
  \hfill
  \begin{minipage}[b]{0.49\textwidth}
    \label{tab:datasets}
    \centering
      \resizebox{1\textwidth}{!}{
    \begin{tabular}{cccc}
    Data set & Size  & Data Characteristics \\
    \toprule
    Uniform (synthetic)   & 10000  & Continuous\\
    \midrule
    Gaussian (synthetic)  & 10000  &  Continuous \\
    \midrule
    Goodreads rating  & 11123  &  Continuous \\
    \midrule
    Goodreads pages & 11123  &  Continuous \\
    \midrule
    Adult hours  & 48842  &  96 Categories \\
    \midrule
    Adult age  & 48842  &  74 Categories
      \end{tabular}
      }
      \captionof{table}{Data set properties.}
    \end{minipage}
  \end{minipage}
\end{center}

\subsection{Empirical Error Analysis}\label{sec:accuracy}
We compare the error of  \algoname~and the baseline algorithms. Our error metric is the average gap of the approximate quantiles $V=(v_1,\ldots,v_m)$ and and the true ones
$O=(o_1,\ldots,o_m)$:
\[
\frac{\sum_{j \in [m]} \mathrm{Gap}_X(o_j, v_j)}{m} \ .
\]

\paragraph{DP Error Analysis:} We randomly chose $1000$ samples from each dataset and checked the error of each algorithm with
$m=1$ to $m=120$
uniform quantiles in the range $[-100, 100]$.
We used the privacy parameter $\eps = 1$. This process was repeated $100$ times. Figure~\ref{fig:accuracy} shows the average of the error across the $100$ iterations. Figure~\ref{fig:accuracy_zoom_in} zooms in on the error for $m=1,\ldots,35$ quantiles.
\algoname~performs better than the baselines
almost in all experiments, except for a few small values of $m$ where the performance of 
$\mathrm{JointExp}$
 was slightly better.
As the number of quantiles increases
\algoname~wins by a larger margin.

\paragraph{zCDP Error Analysis:} As  in previous experiments we randomly chose $1000$ samples from each dataset and checked the error of each algorithm
for $m=1,\ldots, 120$ 
uniform quantiles in the range $[-100, 100]$. 
All algorithms were $\rho$-zCDP for  $\rho = \frac{1}{2}$.
For this we used 
$\eps'=1/\sqrt{m}$ in each application of the exponential mechanism by $\mathrm{AppindExp}$,
and a Laplace noise of magnitude 
$\eps'=\sqrt{h}$ in each node of
the tree computed by $\mathrm{AggTree}$. 
In each
application of the exponential mechanism by \algoname~we used
$\eps'=\sqrt{1/(\log m +1)}$ as described in 
Theorem \ref{lem:uqu_zcdp}.
The algorithm \textbf{JointExp} with
$\eps=1$ is $\frac{1}{2}$-zCDP by Lemma 
\ref{lem:cdp_comp}, so its error is the same as in the previous experiment.
 Figure~\ref{fig:accuracy_cdp} shows the average of the error for z-CDP across the $100$ iterations. Figure~\ref{fig:accuracy_cdp_zoom_in} zooms in on the error for
 $m=1,\ldots,35$.
 \algoname~performs much better than the baselines even for  small number of quantiles.

\begin{figure}[H]
\begin{center}
  \includegraphics[width=1\linewidth]{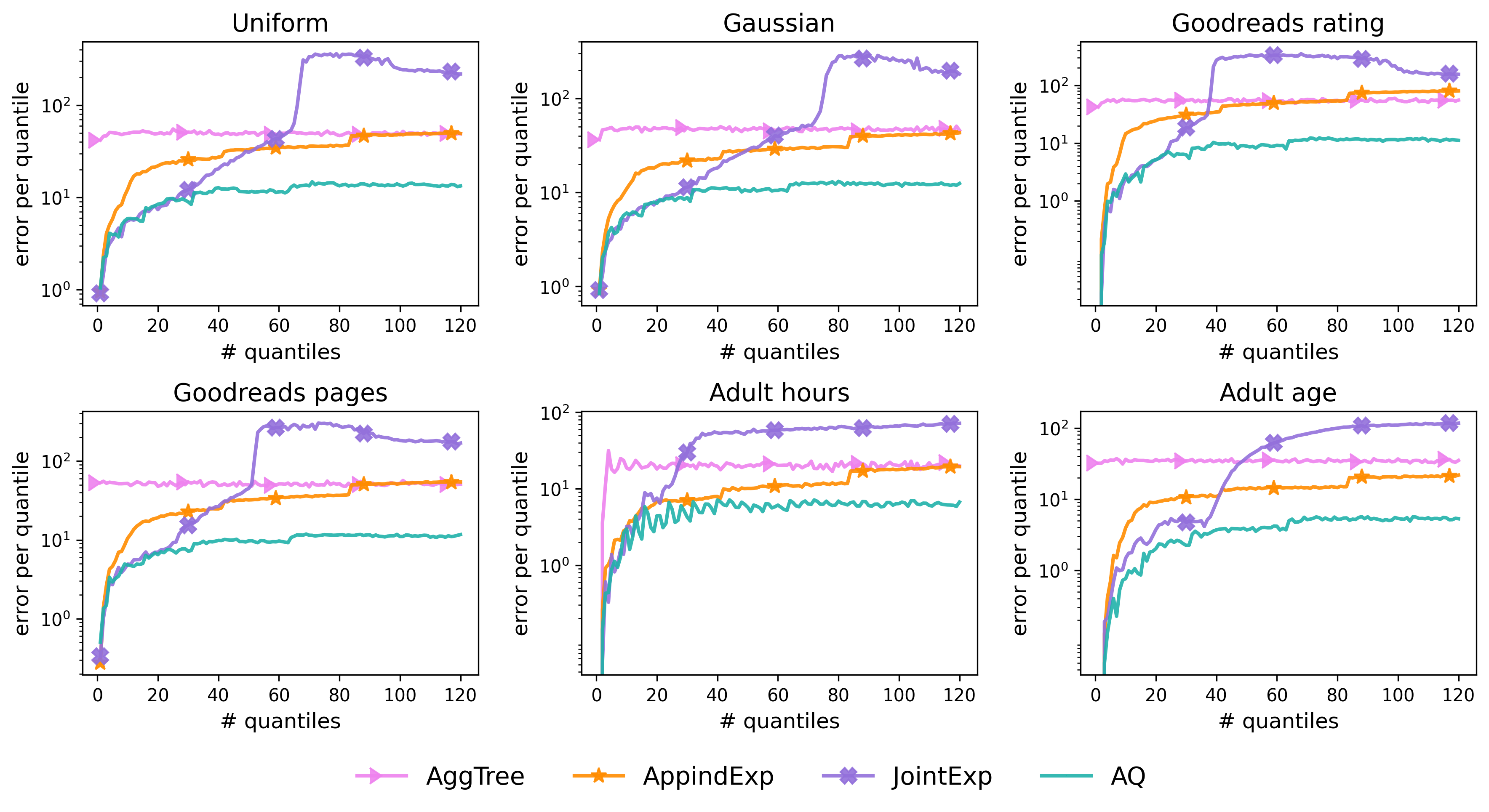}
  \caption{The $x-$axis shows the number of quantiles and the $y-$axis shows the gap per quantile averaged over 100 trials with $\eps = 1$. Note that the graphs are in log scale.}
  \label{fig:accuracy}
\end{center}
\end{figure}

\begin{figure}[H]
\begin{center}
  \includegraphics[width=1\linewidth]{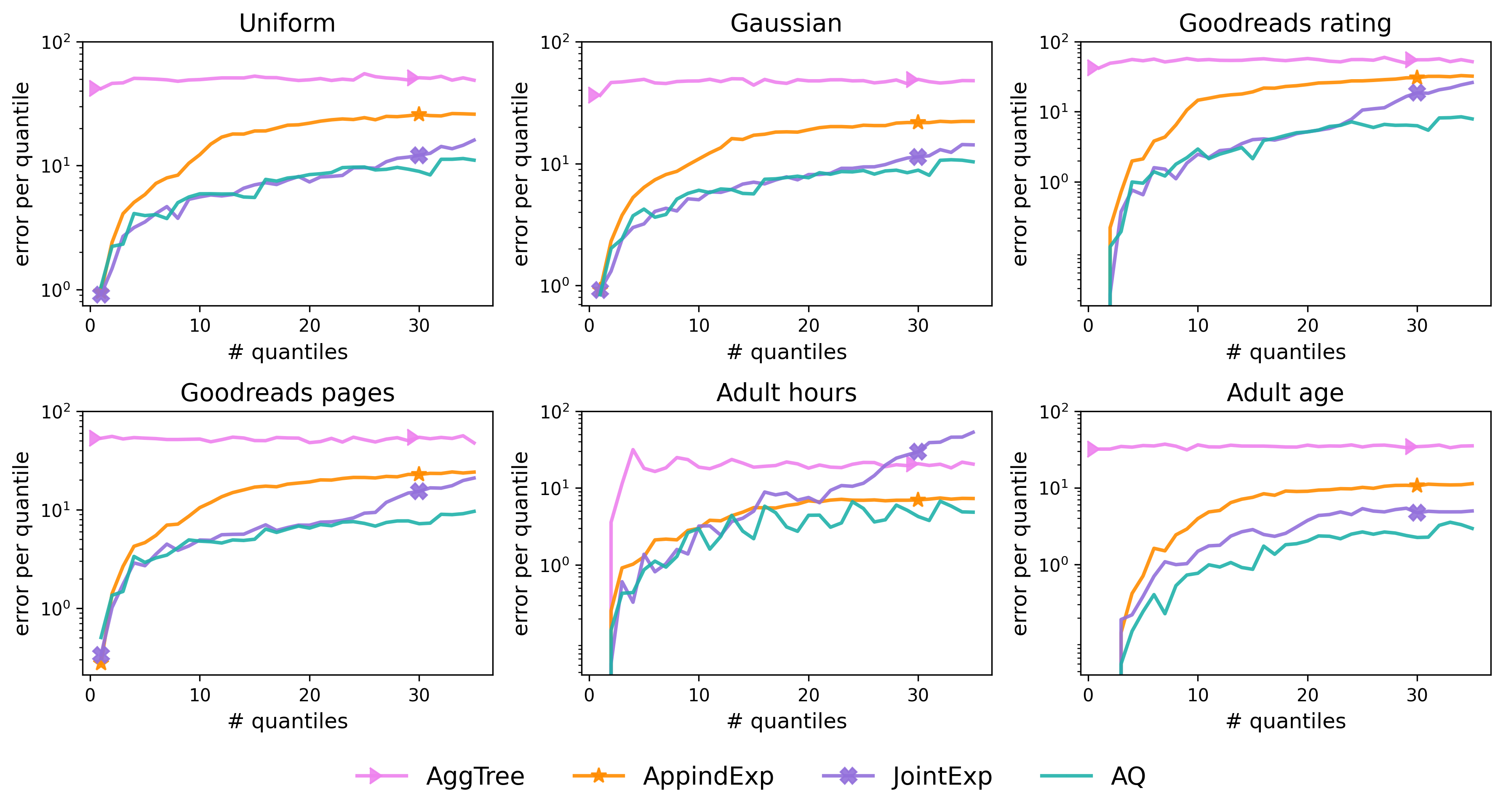}
  \caption{The $x-$axis shows the number of quantiles and the $y-$axis shows the gap per  quantile averaged over 100 trials with $\eps = 1$. Note that the graphs are in log scale.}
  \label{fig:accuracy_zoom_in}
\end{center}
\end{figure}

\begin{figure}[H]
\begin{center}
  \includegraphics[width=1\linewidth]{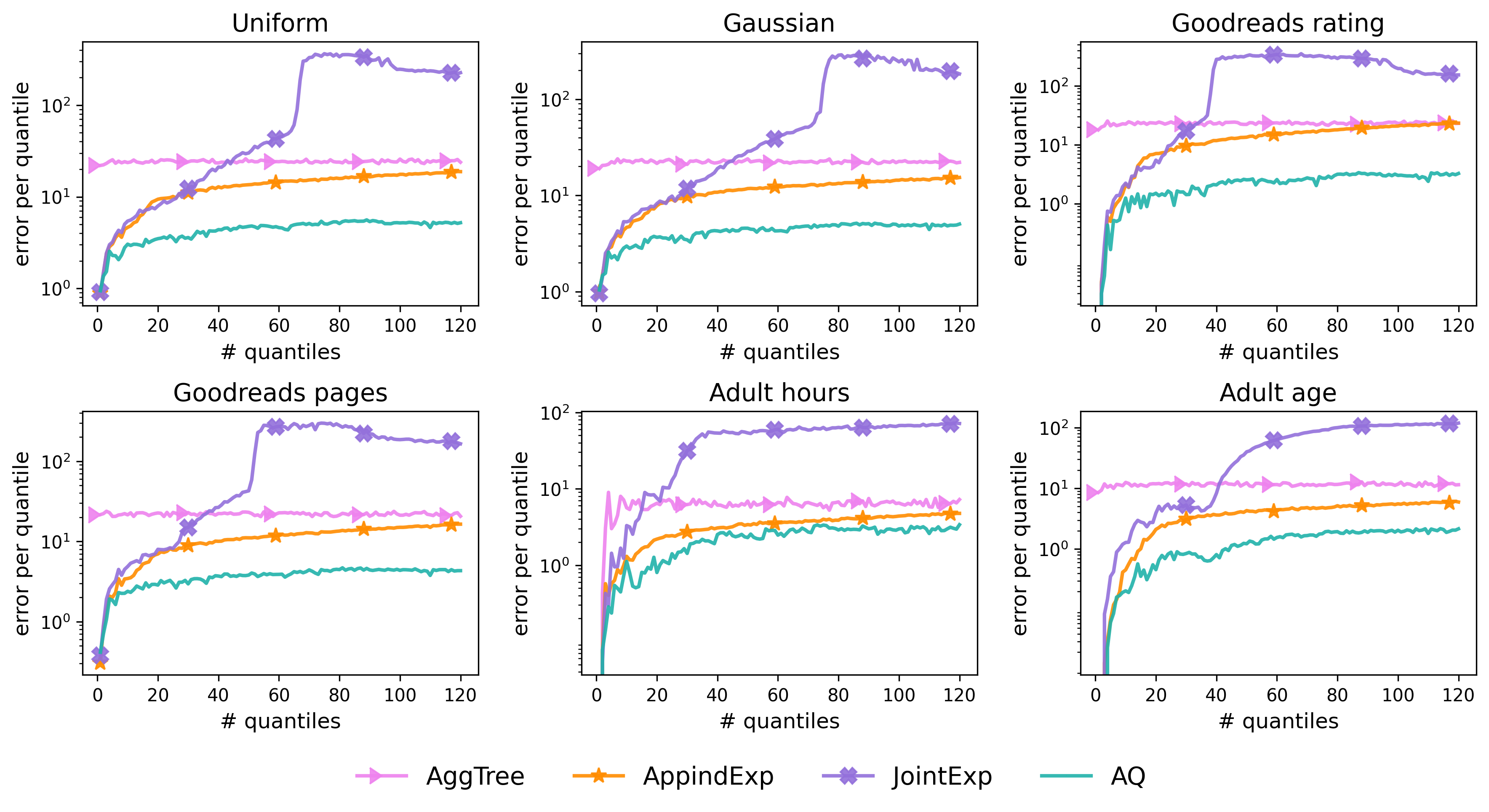}
  \caption{zCDP: The $x-$axis shows the number of quantiles and the $y-$axis shows the gap per  quantile averaged over 100 trials with $\rho = 1/2$. Note that the graphs are in log scale.}
  \label{fig:accuracy_cdp}
\end{center}
\end{figure}

\begin{figure}[H]
\begin{center}
  \includegraphics[width=1\linewidth]{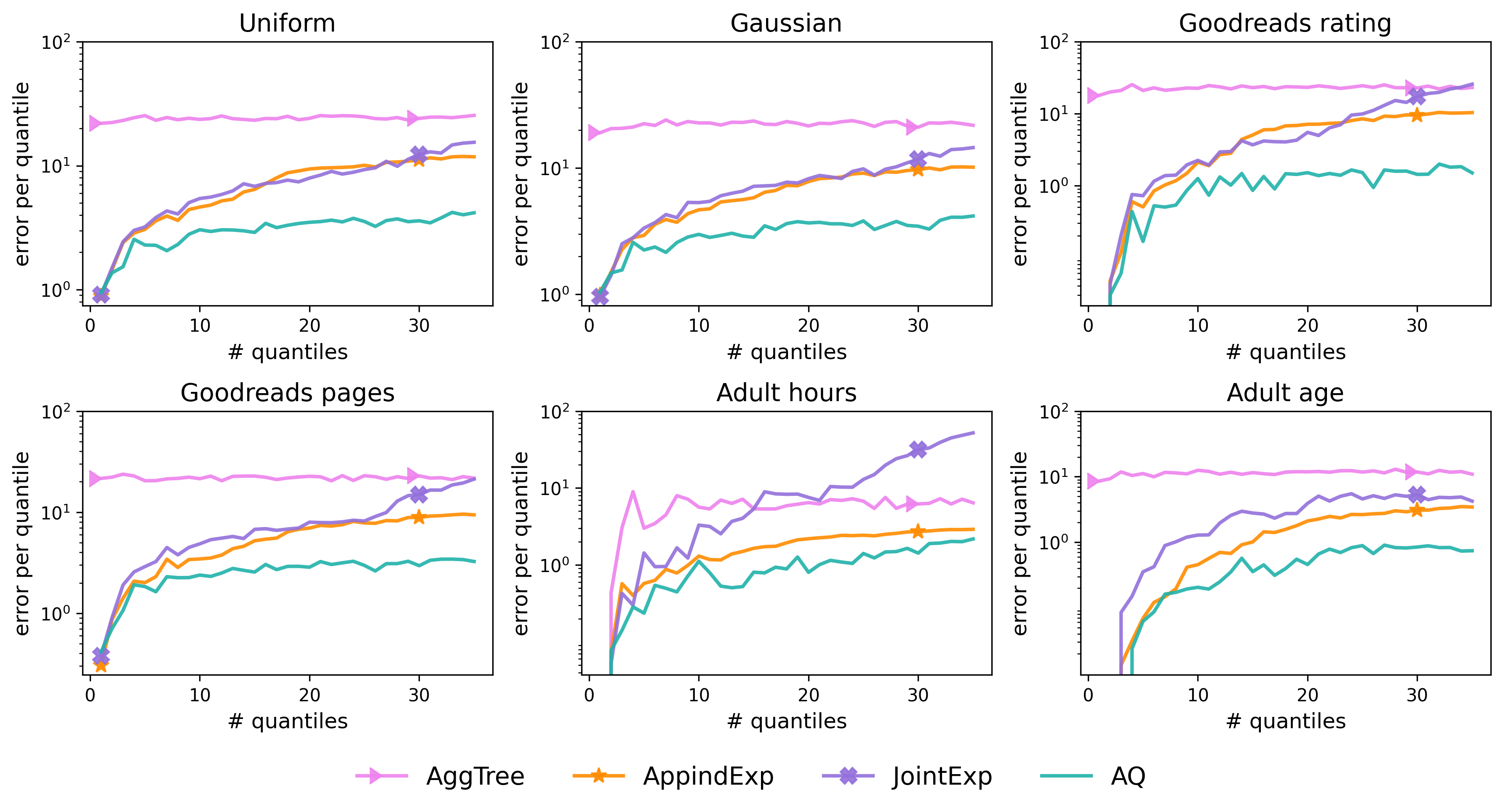}
  \caption{zCDP: The $x-$axis shows the number of quantiles and the $y-$axis shows the gap per  quantile  averaged over 100 trials with $\rho = 1/2$. Note that the graphs are in log scale.}
  \label{fig:accuracy_cdp_zoom_in}
\end{center}
\end{figure}

\subsection{Time Complexity Experiments}\label{sec:time}
Given a sorted dataset,
it takes $O(n)$
time to find all quantiles at a single level of the recursion tree of \algoname. Therefore the overall time complexity (\textit{i.e.,} without the sort) of the \algoref~is $O(n\log{m})$, where $m$ is the number of quantiles. In comparison, the baseline algorithms are computationally more expensive: $\mathrm{JointExp}$ algorithm runs in $O(mn \log{n} + m^2n)$ time, $\mathrm{AppindExp}$ algorithm runs in $O(mn)$ time and $\mathrm{AggTree}$ algorithm runs in $O(n\log{n})$ time. We empirically compared the running time of \algoname~to the running times of the baseline algorithms. For each dataset we measured the time required to find $m \in [120]$ quantiles in a  sub-sample of size $1000$ of each dataset, averaged over $100$ trials per dataset. Figure~\ref{fig:time} shows the average running time across all datasets (Section~\ref{sec:datasets}), each experiment used one core of an Intel i9-9900K processor.
We see that the running time of
\algoname~ is about ten times smaller than of $\mathrm{AggTree}$ and at least a $100$ smaller than of $\mathrm{JointExp}$  and $\mathrm{AppindExp}$ .

\begin{figure}[H]
\begin{center}
  \includegraphics[width=0.3\linewidth]{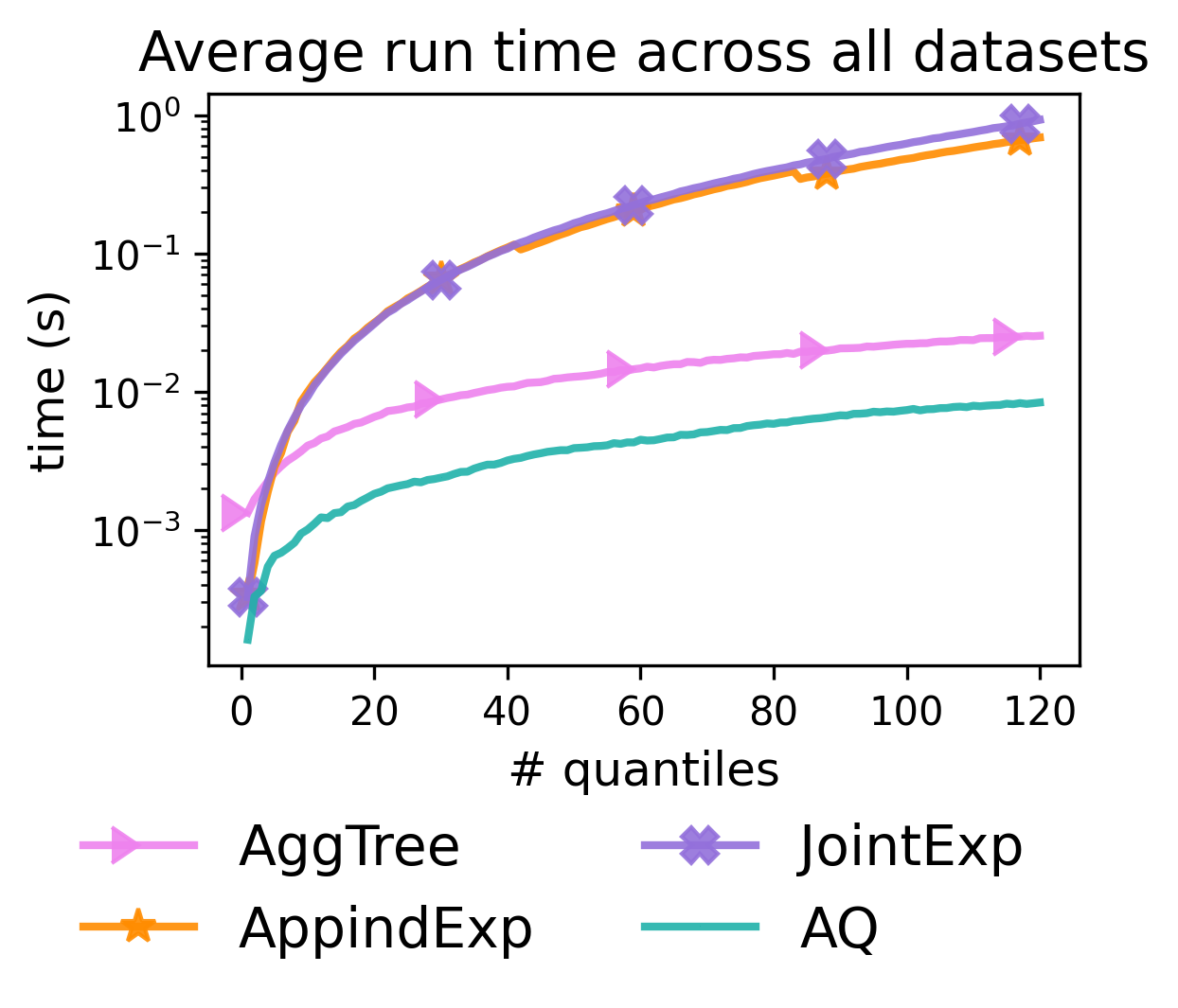}
  \caption{$x-$axis number of quantiles, $y-$axis is the run time averaged across all datasets, each dataset averaged across $100$ trials. Note that the graphs are in log scale.}
  \label{fig:time}
\end{center}
\end{figure}

\bibliographystyle{plainnat}

\newpage

\begin{center}
  {\bf \Large Appendix}  
\end{center}
\appendix

\section{DP Single quantile}\label{app:exp_dp}
In the DP single quantile problem, the input is a single quantile $q$ and a database $X\in (a,b)^n$. The output is a quantile estimate $v \in (a, b)$ such that $\Pr_{x \sim U_X} [x \leq v ] \approx q$ (in a sense that Lemma \ref{lem:expu} would make precise). We solve this problem  using the exponential mechanism of \cite{mcsherry2007mechanism} on $(a,b)$ with the utility function:
\[
u(X, w) =  -\left||\{x \in X | x < w\}| - \lfloor q \cdot n \rfloor \right|  = -\mathrm{Gap}_{X}(o,w),
\]
where $o$ is a $q$-quantile and $w\in (a,b)$. This mechanism  samples each point $w\in (a,b)$ to be the output with density proportional to $\exp(\eps u(X,w)/2)$. Note that 
the sensitivity of $u$ is $1$, that is 
$\max\{ |u(X, \omega)- u(X', \omega)|\} \le 1$, where the maximum is over neighboring datasets $X$ and $X'$ and points $\omega\in (a,b)$. 
The  largest utility is of a $q$-quantile and equals to $0$.

We can sample from this distribution using the technique given by \cite{smith2011privacy} (see their Algorithm 2). The idea is to split the sampling process into two steps:
\begin{enumerate}
    \item
    Let $I_k = [x_{k-1},x_k]]$, $k=1,\ldots, n+1$
    where $x_0=a$, $x_{n+1}=b$, be the set of $n+1$ intervals between data points. We sample an interval from this set of intervals,
 where the probability of sampling $I_k$ is proportional to
    \[
    \Pr[\calA(X) = I_k] \propto \exp\left(\frac{\eps u(X,I_k)}{2}\right) \cdot (x_k - x_{k-1}) ,
    \]
Note that all points in $I_k$ have the same utility which we denote by  $u(X,I_k)$.
    \item
Return a uniform random point from the sampled interval.
\end{enumerate}

\begin{lemma}
\label{lem:expu}
Given dataset $X \in (a,b)^n$ and quantile $q \in [0, 1]$, the exponential mechanism is $\eps$-DP, and with probability  $1 - \beta$ outputs $v$ that satisfies
\[
\mathrm{Gap}_X(o,v) \leq  2 \cdot \frac{\log{\psi}  - \log{\beta}}{\eps},
\]
where $o$ is a true $q$-quantile and $\psi = \frac{b-a}{\min_{k \in [n+1]} \{x_k - x_{k - 1} \}}$.
\end{lemma}

\begin{proof}
$\eps$-DP follows in a straightforward way by bounding the ratio of the densities of a point $w$ in the destribution defined by $X$ and in the distribution defined by $X'$;
See also \cite{dwork2014algorithmic}. 

Let  $I_t$ be an interval such that $u(X, I_t) \leq -\gamma$. It follows that the probability of sampling  a point from $I_t$ is at most
\[
\Pr[\calA(X) = I_t] \leq \frac{\exp\left(\frac{-\eps \gamma}{2}\right) \cdot (x_t-x_{t-1})}{\sum_{k \in [n + 1]} \exp\left(\frac{\eps u(X,I_k)}{2}\right) \cdot (x_k - x_{k-1})} \ .
\]
Using the union bound we get that:
\begin{eqnarray*}
\Pr[\calA(X) \leq  - \gamma] 
&\leq  &
 \frac{ \exp\left(\frac{-\eps \gamma}{2}\right)(b-a)}{\sum_{k \in [n + 1]} \exp\left(\frac{\eps u(X,I_k)}{2}\right) \cdot (x_k - x_{k-1})} \\
& \leq &
 \frac{\exp\left(\frac{-\eps \gamma}{2}\right)(b-a)}{\exp\left(\frac{\eps u(X,I_o)}{2}\right)(x_o - x_{o-1})}  \\
& \leq &
 \frac{b-a}{\min_{k\in[n + 1]}(x_k - x_{k-1})} \cdot  \exp\left(\frac{-\eps \gamma}{2}\right) \\
& = &
 \psi\exp\left(\frac{-\eps \gamma}{2}\right)
\end{eqnarray*}
where $I_o$ is the interval containing the $q$-quantiles so 
$u(X, I_0)=0$. 
It follows that with probability less than $\beta$, we sample an interval whose utility is at
most $-\gamma$ for 
\[\gamma=2 \Delta_{u} \cdot \frac{\log{\psi} - \log{\beta}}{\eps}.\]
Since in the second step we sample the $q$-quantile $v$ uniformly from the interval selected in the first step, we get that with probability $1 - \beta$ the output $v$ satisfies:
\[
\mathrm{Gap}_X(o,v) \leq  2 \cdot \frac{\log{\psi}  - \log{\beta}}{\eps}.
\]
\end{proof}

\ifx
Since we run the exponential mechanism over the  domain $\calD$, by Lemma~\ref{lem:exp_dp} we get that this scheme is $\eps$-DP and that the approximation  $v$ which we draw satisfies the  utility guarantee:
\[
\Pr_{v \in \calD}[u(X, v) > \mathrm{OPT}-\gamma] \leq |\calD|\cdot \exp \left(-\frac{\eps \gamma}{2 \Delta_{\mathrm{Gap}}} \right),
\]
where $\mathrm{OPT}=u(X, o)=0$ is the utility of a true $q$-quantile, and 
$\Delta_{\mathrm{Gap}}$ is $\max\{ |u(X, \omega)- u(X', \omega)|\}$ over pairs $X$ and $X'$ of neighboring datasets and elements $\omega\in \calD $.
It follows that with probability
$\beta/m$, $u(X,v) = \mathrm{Gap}_X(o,v)$ is at most:
\[\gamma=2 \Delta_{\mathrm{Gap}} \cdot \frac{\log{\psi} +\log{m} - \log{\beta}}{\eps}.\]
Substituting
$\Delta_{\mathrm{Gap} }\le 1$
we get that:
\[\gamma \leq 2 \cdot \frac{\log{\psi} +\log{m} - \log{\beta}}{\eps}.\]

\begin{lemma}
\label{lem:expu}
Given dataset $X \in \calD^n$ and quantile $q \in [0, 1]$, the exponential mechanism is $\eps$-DP, and with probability  $1 - \frac{\beta}{m}$ output $v$ that satisfies
\[
\mathrm{Gap}_X(o,v) \leq  2 \cdot \frac{\log{\psi} +\log{m} - \log{\beta}}{\eps}
\]
where $o = \lfloor q \cdot n \rfloor$
\end{lemma}
\fi

\end{document}